%% file: dpnc.tex
\documentclass[11pt]{article}
\pdfoutput=1
\usepackage{fullpage}
\usepackage{times}
\usepackage{graphicx,color}
\usepackage{array,float}
\usepackage{url}
\usepackage[usenames,dvipsnames]{xcolor}
\usepackage{amstext,amssymb,amsmath}
\usepackage{amsthm}
\usepackage{mathtools}
\usepackage{verbatim}
\usepackage{bm}
\usepackage{paralist}
\usepackage{ulem}
\usepackage[numbers]{natbib}
\usepackage{hyperref}
\usepackage[capitalize,nameinlink]{cleveref}
\usepackage[noend]{algorithmic}
\usepackage{algorithm}
\usepackage{xspace}
\usepackage{authblk}
\usepackage{booktabs}

\newif\ifnotes\notestrue

\ifnotes
\usepackage[textsize=scriptsize,textwidth=2cm]{todonotes}
\else
\usepackage[disable]{todonotes}
\fi

\newtheorem{lem}{Lemma}

\newtheorem{thm}[lem]{Theorem}
\newtheorem{cor}[lem]{Corollary}

\newtheorem{defn}[lem]{Definition}

\newtheorem{proposition}[lem]{Proposition}

\crefname{thm}{Theorem}{Theorems}

\newcommand{\eps}{\varepsilon}
\newcommand{\Ex}{\mathop{\mathbb{E}}}

\DeclareMathOperator*{\argmin}{arg\,min}

\DeclareMathOperator*{\esssup}{ess\,sup}

\renewcommand{\Re}{\mathbb{R}}

\newcommand{\N}{\mathcal{N}}

\newcommand{\Id}{\mathbb{I}}

\newcommand{\A}{\mathcal{A}}

\newcommand{\bE}{\mathbb{E}}

\newcommand{\cK}{{\mathcal{K}}}

\newcommand{\cN}{\mathcal{N}}
\newcommand{\cO}{\mathcal{O}}
\newcommand{\cP}{\mathcal{P}}

\newcommand{\cX}{\mathcal{X}}

\newcommand{\cZ}{\mathcal{Z}}

\newcommand{\ed}{\ensuremath{(\eps,\delta)}}

\newcommand{\ud}{\mathrm d}

\newcommand{\normal}[2]{\mathcal{N}(#1, #2)}
\newcommand{\conv}{\ast}
\newcommand{\Renyi}{R\'enyi\,}

\DeclarePairedDelimiterX{\infdivx}[2]{(}{)}{#1\;\delimsize\|\;#2}
\newcommand{\Dalpha}[2]{\mathrm{D}_{\alpha}\infdivx*{#1}{#2}}
\newcommand{\Dpalpha}[3]{\mathrm{D}_{#1}\infdivx*{#2}{#3}}
\newcommand{\Dzalpha}[2]{\mathrm{D}^{(z)}_{\alpha}\infdivx*{#1}{#2}}
\newcommand{\Dzpalpha}[3]{\mathrm{D}^{(#1)}_{\alpha}\infdivx*{#2}{#3}}

\newcommand{\eff}{\ensuremath{\psi}}
\newcommand{\efs}{\ensuremath{\{\psi_t\}}}
\newcommand{\efps}{\ensuremath{\{\psi'_t\}}}

\newcommand{\zetas}{\ensuremath{\{\zeta_t\}}}
\newcommand{\K}{\mathcal{K}}

\providecommand{\ignore}[1]{\relax}

\newcommand{\CNI}{\ensuremath{\mbox{CNI}}\xspace}

\begin{document}
\title{Privacy Amplification by Iteration}\author[1]{Vitaly Feldman
}
\author[1]{
Ilya Mironov
}
\author[1]{
Kunal Talwar
}
\author[1,2]{Abhradeep Thakurta
}

\affil[1]{\small Google Brain, \texttt{\{vitalyfm,mironov,kunal\}@google.com}.}
\affil[2]{\small UC Santa Cruz, \texttt{aguhatha@ucsc.edu}.}

\date{}
\maketitle

\begin{abstract}
	Many commonly used learning algorithms work by iteratively updating an intermediate solution using one or a few data points in each iteration.  Analysis of differential privacy for such algorithms often involves ensuring privacy of each step and then reasoning about the cumulative privacy cost of the algorithm. This is enabled by composition theorems for differential privacy that allow releasing of all the intermediate results.  In this work, we demonstrate that for contractive iterations, not releasing the intermediate results strongly amplifies the privacy guarantees.

  We describe several applications of this new analysis technique to solving convex optimization problems via noisy stochastic gradient descent. For example, we demonstrate that a relatively small number of non-private data points from the same distribution can be used to close the gap between private and non-private convex optimization. In addition, we demonstrate that we can achieve guarantees similar to those obtainable using the privacy-amplification-by-sampling technique in several natural settings where that technique cannot be applied.

\end{abstract}
\thispagestyle{empty}
\pagebreak
\setcounter{page}{1}

\input{introduction}
\input{related}
\input{prelims}

\input{coupled_descent}
\input{privacy}
\input{applications}

\section*{Acknowledgements}
We thank \'{U}lfar Erlingsson and Tomer Koren for useful suggestions and insightful discussions of this work.

\bibliographystyle{alpha}
\bibliography{reference}

\appendix
\input{app_contractivity}
\input{app_bstproof}

\input{app_smoothness}

\end{document}

%% file: introduction.tex
\section{Introduction}
\label{sec:introduction}
Differential privacy~\cite{DMNS06} is a standard concept for capturing privacy of statistical
algorithms. In its original formulation, (pure) differential privacy is parameterized by a single real number---the so-called privacy budget---which characterizes the privacy loss of an individual contributor to the input dataset.

As applications of differential privacy start to proliferate, they bring to the fore the problem of administering the privacy budget, with specific emphasis on \emph{privacy composition} and \emph{privacy amplification}.

Privacy composition enables modular design and analysis of complex and heterogeneous algorithms from simpler building blocks by controlling the total privacy budget of their combination. Improving on ``na\"ive'' composition, which simply (but very consequentially!) states that the privacy budgets of composition blocks sum up, ``advanced'' composition theorems allow subadditive accumulation of the privacy budgets. All existing proofs of advanced composition theorems assume that all intermediate outputs are revealed, whether the composite mechanism requires it or not.


Privacy amplification goes even further by bounding the privacy budget---for select mechanisms---of a combination to be \emph{less} than the privacy budget of its parts. The only systematically studied instance of this phenomenon is \emph{privacy amplification by sampling}~\cite{KLNRS08, beimel2014bounds,privacy-for-free,tCDP,subsampled-RDP,DLDP}. In its basic form, for $\eps = O(1)$, an $\eps$-differentially private mechanism applied to a secretly sampled $p$ fraction of the input satisfies $O(p\eps)$-differential privacy. More recent results demonstrate that privacy can be amplified in proportion to $p^2$ (for a Gaussian additive noise mechanism and appropriate relaxations of differential privacy).

This work introduces a new amplification argument---\emph{amplification by iteration}---that in certain contexts can be seen as an alternative to privacy amplification by sampling. As an exemplar of the kind of algorithms we wish to analyze, we consider noisy stochastic gradient descent for a smooth and convex objective.

Our preferred privacy notion for formally stating our contributions is \Renyi differential privacy (RDP). For the purpose of this introduction, it suffices to keep in mind that RDP is parameterized with $1<\alpha\leq \infty$ and measures the \Renyi divergence of order $\alpha$ (denoted $\mathrm{D}_\alpha$) between the output distributions of a randomized algorithm on two neighboring datasets. It is a relaxation of (pure) differential privacy which has been instrumental for achieving tighter bounds on privacy cost in a number of recent papers on privacy-preserving machine learning. In addition, to being a privacy definition in its own right, one can easily translate RDP bounds to usual $(\eps, \delta)$-DP bounds.

Our first contribution is a general theorem that states that, under certain conditions on an iterative process, the process shrinks the \Renyi divergence between distributions. 
We will focus on the simplest form of these conditions in which the mechanism is a composition of a sequence of \emph{contractive} (or $1$-Lipschitz) maps and an additive Gaussian noise mechanism. This is a natural setting for several differentially private optimization algorithms. A more general treatment that allows other Banach spaces and noise distributions appears in \cref{ss:main-result}.
\begin{thm} [Informal]
Let $x_0 \in \Re^d$ and let $X_T$ be obtained from $x_0$ by iterating
\[
x_{t+1} \doteq \eff_{t+1}(x_t) + Z_{t+1}
\]
for some sequence of contractive maps $\{\eff_t\}^T_{t=1}$ and $Z_{t+1}\sim \normal{0}{\sigma^2 \Id_d}$. Let $X'_T$ denote the output of the same process started at some $x'_0$. Then for every $\alpha \geq 1$, $\Dalpha{X_T}{X'_T} \leq \frac{\alpha \|x_0-x'_0\|}{2T\sigma^2}$.
\end{thm}

We note that in this result we measure the divergence only between the final steps, in other words, the intermediate steps of the iteration are not revealed. This theorem is a special case of our more general result~\cref{thm:pai-general}.
This result translates a {\em metric} assumption of bounded distance between $x_0$ and $x'_0$ to an information-theoretic conclusion of bounded \Renyi divergence between $X_T$ and $X'_T$. While standard facts about the Gaussian distribution allow one to make such a statement for a one-step process, the intermediate arbitrary contractive steps essentially rule out a first principles approach to proving such a theorem. We use a careful induction argument that rests on controlling the ``distance'' between $X_t$ and $X'_t$.  We start by measuring the metric distance when $t=0$ and gradually transform this to an information theoretic divergence at $t=T$. We interpolate between these two using a new  {\em hybrid} distance measure that we refer to as {\em shifted divergence}. We believe that this notion should find additional applications in the analyses of stochastic processes. Our bounds are tight (with no loss in constants) and show that the worst-case for such a result is when all the contractive maps are the identity map.

This result has some surprising implications. Consider an iterative mechanism that processes one input record at a time, $n$ iterations in total. The immediate application of this result to this mechanism leads to the following observation about individuals' privacy loss. The person whose record was processed last experiences privacy loss afforded by the Gaussian noise added at the last iteration. At the same time, the person whose record was processed \emph{first} suffers the least amount of privacy loss, equal to $1/n$ of the last one's. Importantly, the order in which the inputs were considered need not be random or secret for this analysis to be applicable. In contrast, privacy amplification by sampling depends crucially on the sample's randomness and secrecy.

We outline some applications of this analysis in privacy-preserving machine learning via convex optimization.

\paragraph{Distributed stochastic gradient descent.} In this setting records are stored locally, and the parties engage in a distributed computation to train a model~\cite{DKMMN06}. Using amplification by sampling as in DP-SGD by Abadi et al.~\cite{DLDP}  would require keeping secret the set of parties taking part in each step of the algorithm. When the communication channel is not trusted, hiding whether or not a party takes part in a certain step would essentially require all parties to communicate in all steps, leading to an unreasonable amount of communication. In addition, the assumption that the sample of parties participating in each step is a random subset may itself be difficult to enforce in many settings.

Our approach does not need the order of participating parties to be random or hidden. It is sufficient to hide the model itself until a certain number of update steps are applied. This approach then allows significantly reducing communication costs to be proportional to the size of the mini-batch (the number of records consumed by each update).  Additionally, our approach can amplify privacy even when the noise added in each step is too small to guarantee much privacy. This is in contrast to amplification by sampling, which requires the unamplified privacy cost to be small to start with: a starting $\eps$ becomes $\approx q\eps(1+\exp(\eps))$ which is close to $2q\eps$ for small $\eps$ but grows quickly, and for instance, precludes setting $\eps \geq 1/q$ for small $q$.  Our main result applies for arbitrary $\sigma$ so that even if each $\sigma$ is very small (say, $1/\sqrt{n}$) the final privacy is non-vacuous. A smaller noise scale then permits a smaller size of each mini-batch, further reducing the communication cost. 
On the negative side, the privacy guarantee we get varies between examples: examples used early in the SGD get stronger privacy than those occurring late.

\paragraph{Multi-query setting.} Our approach above gives better privacy than competing approaches to the parties taking part early in the computation, while giving similar guarantees to the last user. This better per-user privacy guarantee can allow one to solve several such convex optimization problems on the same set of users, at no increase in the worst-case privacy cost. Specifically, if we have $n$ parties, then we can solve $\tilde{\Omega}(n)$ such convex optimization problems at the same privacy cost as answering one of them. More generally, the privacy cost grows linearly in $\tilde{O}(\sqrt{\max\{k/n,1\}})$. 
To our knowledge, except for privacy-amplification-by-sampling, existing techniques such as output perturbation have utility bounds that grow linearly in $\sqrt{k}$.

\paragraph{Public/private data.} The setting in which some public data from the same distribution as private data is available has been recently identified as promising and practically important~\cite{PATE,Blender}. The public corpus can be based on opt-in population, such as  a product's developers or early testers, data shared by volunteers~\cite{church2005personal}, or be released through a legal process~\cite{enron}.

In this model, the last iterations of the iterative algorithm can be done over the public samples whose privacy need not be preserved. Since data points used early lose less privacy, we can add much less noise at each step. In effect, having $m$ public samples decreases the error due to the addition of noise by a factor of $\sqrt{m}$. In the absence of public data, privacy comes at a provable cost: while the statistical error due to sampling scales as $1/\sqrt{n}$ independently of the dimension, the error of the differentially private version scales as $\sqrt{d}/\sqrt{n}$ \cite{bassily2014differentially}. Our results imply that for convex optimization problems satisfying very mild smoothness assumptions, given $\tilde{O}(d)$ public data points, we can ensure that the additional error due to privacy is comparable to the statistical error.

We remark that our technique requires that the optimized functions satisfy a mild smoothness assumption. However, as we show, in our applications we can always achieve the desired level of smoothness by convolving the optimized functions with the Gaussian kernel. Such convolution introduces an additional error but this error is dominated by the error necessary to ensure privacy.

\paragraph{Organization.} The rest of the paper is organized as follows. After discussing some additional related work, we start with some preliminaries in \cref{sec:prelims}. We present our main technique in \cref{sec:coupled_descent}. \Cref{sec:privacy} shows how this technique can be applied to versions of the noisy stochastic gradient descent algorithm. Finally, in \cref{sec:applications}, we apply this framework to derive the applications mentioned above. 

%% file: related.tex
\subsection{Related Work}
\label{sec:related}

The field of differentially private convex optimization spans almost a decade~\cite{CM08,CMS11,jain2012differentially,KST12,ST13sparse,DuchiJW13,ullman2015private,jain2014near,BassilyST14,talwar2015nearly,smith2017interaction,wu2017bolt,INSTTW19}. Many of these results are optimal under different regimes such as empirical loss, population loss, the low-dimensional setting ($d\ll n$) or the high-dimensional setting $d\gg n$. Some of the algorithms (e.g., output perturbation \cite{CMS11} and objective perturbation \cite{CMS11,KST12}) require finding a global optimum of an optimization problem to argue privacy and utility, while the others are based on the variants of noisy stochastic gradient descent. In this section we restrict ourselves to only the population loss, and allow comparisons to algorithms that can be implemented with one pass of stochastic gradient descent over the data set~$S$ for a direct comparison (which is close to the typical application of optimization algorithms in machine learning). We note that our analysis technique also applies to multi-pass and batch versions of gradient descent. In this setting our algorithm achieves close to optimal bounds on population loss (see \cref{tab:bounds} for details).

In this table we also compare the local differential privacy of the algorithms \cite{KLNRS08}. In several settings (such as distributed learning) we want the published outcome of the optimization algorithm to satisfy a strong level of (central) differential privacy while still guaranteeing  $\eps_{\sf local}$ differential privacy. Local differential privacy protects the user's data even from the aggregating server or an adversary who can obtain the complete transcript of communication between the server and the user.

\begin{table*}[h!]
	\begin{center}
		\begin{tabular}{lccc}
			\toprule
			& \multicolumn{2}{c}{\textbf{Excess loss}}& \textbf{LDP} ($\eps_{\sf local}$)\\
			\cline{2-3}
			{\bf Algorithm}& \textbf{for one task} &  \textbf{for $k\leq n$ tasks} &  \textbf{for one task}
			\\
			\midrule
			Noisy SGD + sampling$^\dagger$ \cite{BassilyST14} & $\tilde{O}\left(\sqrt\frac{d}{\eps n}\right)$ & $\tilde{O}\left(\sqrt\frac{d}{\eps n}\right)$ &  $\tilde O(n)$\\
			Noisy SGD \cite{DuchiJW13,smith2017interaction} & $\tilde{O}\left(\sqrt\frac{d}{\eps^2 n}\right)$ &$\tilde{O}\left(\sqrt\frac{d k}{\eps^2 n}\right)$ & $\eps$ \\
			Output perturbation$^*$ \cite{CMS11,wu2017bolt} & $\tilde{O}\left(\sqrt\frac{d}{\eps^2 n}\right)$  &$\tilde{O}\left(\sqrt\frac{d k}{\eps^2 n}\right)$ & $\infty$\\
			This work & {\color{blue}$\tilde{O}\left(\sqrt\frac{d}{\eps^2 n}\right)$}  & {\color{blue}$\tilde{O}\left(\sqrt\frac{d}{\eps^2 n}\right)$} & {\color{blue}$\eps$}\\
			\bottomrule
		\end{tabular}
		\caption{The excess loss corresponds to the excess population loss. Comparison for a single pass over the dataset (i.e., at most $n$ gradient evaluations). For brevity, the table hides dependence on ${\sf poly}\ln(1/\delta)$. ($^\dagger$) This bound is not stated explicitly but can be derived by setting the parameters in \cite{BassilyST14} appropriately.
			($^*$) For output perturbation, we used the variant that can be implemented via SGD in a single pass \cite{wu2017bolt}. LDP stands for local differential privacy.}\label{tab:bounds}
	\end{center}
\end{table*}

We note that some architectures may not be compatible with all privacy-preserving techniques or guarantees. For instance, we assume secrecy of intermediate computations, which rules out sharing intermediate updates (which is a standard step in federated learning~\cite{federated}). In contrast, analyses based on secrecy of the sample (e.g., \cite{KLNRS08,DLDP}) require that either data be stored centrally (thus eliminating local differential privacy guarantees) or all-to-all communications.

%% file: prelims.tex
\section{Preliminaries}
\label{sec:prelims}
We recall definitions and tools from the learning theory, probability theory, and differential privacy and define the notion of shifted divergence. In the process we set up the notation that we will use throughout the paper.

\subsection{Convex Loss Minimization}
\label{sec:convOpt}
Let $\cX$ be the domain of data sets, and $\cP$ be a distribution over $\cX$. Let $S=\{x_1,\dots,x_n\}$ be a data set drawn i.i.d.\ from~$\cP$. Let $\K\subseteq\Re^d$ be a convex set denoting the space of all models. Let $f\colon \K\times \cX\to\Re$ be a loss function, which is convex in its first parameter (the second parameter is a data point and dependence on this parameter can be arbitrary).
The excess population loss of solution $w$ is defined as
\[
\bE_{x\sim\cP}\left[f(w,x)\right]-\min_{v\in\K}\bE_{x\sim\cP}\left[f(v,x)\right].
\]
In order to argue differential privacy we place certain assumptions on the loss function. To that end, we need the following two definitions of Lipschitz continuity and smoothness.

\begin{defn}[$L$-Lipschitz continuity]
	A function $f\colon\K\to\Re$ is \emph{$L$-Lipschitz continuous} over the domain $\K\subseteq \Re^d$ if the following holds for all $w, w' \in\K$: $\left|f(w)-f(w)\right|\leq L\left\|w-w'\right\|_2$.
	\label{def:lipCont}
\end{defn}

\begin{defn}[$\beta$-smoothness]
	A function $f\colon\K\to\Re$ is \emph{$\beta$-smooth} over the domain $\K\subseteq \Re^d$ if for all $w, w' \in\K$, $\| \nabla f(w) - \nabla f(w')\|_2 \leq \beta \left\|w-w'\right\|_2$.
	\label{def:smoothness}
\end{defn}

\subsection{Probability Measures}
\label{sec:measure}

In this work, we will primarily be interested in the $d$-dimensional Euclidean space $\Re^d$ endowed with the $\ell_2$ metric and the Lebesgue measure. Our main result holds in a more general setting of Banach spaces.

We say a distribution $\mu$ is \emph{absolutely continuous} with respect to $\nu$ if $\mu(A) = 0$ whenever $\nu(A) = 0$ for all measurable sets $A$. We will denote this by $\mu \ll \nu$.

Given two distributions $\mu$ and $\nu$ on a Banach space $(\cZ,\|\cdot\|)$, one can define several notions of distance between them. The first family of distances we consider is independent of the norm:
\begin{defn}[\Renyi Divergence~\cite{Renyi61}]\label{def:renyi}
Let $1<\alpha<\infty$ and $\mu, \nu$ be measures with $\mu \ll \nu$. The \Renyi divergence of order $\alpha$ between $\mu$ and $\nu$ is defined as
\[
\Dalpha{\mu}{\nu} \doteq \frac{1}{\alpha-1} \ln \int \left(\frac{\mu(z)}{\nu(z)}\right)^{\alpha} \nu(z)\, \ud z.
\]
Here we follow the convention that $\frac 0 0 = 0$. If $\mu \not\ll \nu$, we define the \Renyi divergence to be $\infty$. \Renyi divergence of orders $\alpha=1,\infty$ is defined by continuity.
\end{defn}

\begin{proposition}[\cite{EH07-Renyi}]\label{prop:renyi}
The following hold for any $\alpha \in (1, \infty)$, and distributions $\mu, \mu', \nu, \nu'$:
  \begin{description}
\item[Additivity:] $\Dalpha{\mu \times \mu'}{\nu \times \nu'} = \Dalpha{\mu}{\nu} + \Dalpha{\mu'}{\nu'}$.
\item[Post-Processing:] For any (deterministic) function $f$, $\Dalpha{f(\mu)}{f(\nu)} \leq \Dalpha{\mu}{\nu}$, where we $f(\mu)$ denotes the distribution of $f(X)$ where $X \sim \mu$.
  \end{description}
\end{proposition}
As usual, we denote by $\mu \conv \nu$ the convolution of $\mu$ and $\nu$, that is the distribution of the sum $X+Y$ where we draw $X \sim \mu$ and $Y \sim \nu$ independently.

We will also need the following ``norm-aware" statistical distance:
\begin{defn}[$\infty$-Wasserstein Distance]
The \emph{$\infty$-Wasserstein distance} between distributions $\mu$ and $\nu$ on a Banach space $(\cZ,\|\cdot\|)$ is defined as
\[
W_{\infty}(\mu, \nu) \doteq \inf_{\gamma \in \Gamma(\mu, \nu)} \esssup_{(x,y)\sim \gamma} \|x-y\|,
\]
where $(x,y)\sim \gamma$ means that the essential supremum is taken relative to measure $\gamma$ over $\cZ \times \cZ$ parameterized by $(x,y)$.
Here $\Gamma(\mu, \nu)$ is the collection of couplings of $\mu$ and $\nu$, i.e., the collection of measures on $\cZ \times \cZ$ with marginals $\mu$ and $\nu$ on the first and second factors respectively.
\end{defn}

The following is immediate from the definition.
\begin{lem}
\label{lem:winfty_alternate}
The following are equivalent for any distributions $\mu$, $\nu$ over $\cZ$:
\begin{enumerate}
\item $W_\infty(\mu, \nu) \leq s$.
\item There exists jointly distributed r.v.'s $(U, V)$ such that $U ~\sim \mu$, $V \sim \nu$ and $\Pr[\|U-V\| \leq s] = 1$.
\item There exists jointly distributed r.v.'s $(U, W)$ such that $U ~\sim \mu$, $U+W \sim \nu$ and $\Pr[\|W\| \leq s] = 1$.
\end{enumerate}
\end{lem}

Next we define a hybrid\footnote{Here we use a {\em budgeted} version of the definition, putting a hard constraint on the $W_{\infty}$ portion of the distance, as it is most convenient for reasoning about differential privacy. A {\em Lagrangian} version of the definition may be more natural in other applications.} between these two families of distances that plays a central role in our work.
\begin{defn}[Shifted \Renyi Divergence]\label{def:shifted-rdp}
Let $\mu$ and $\nu$ be distributions defined on a Banach space $(\cZ,\|\cdot\|)$. For parameters $z \geq 0$ and $\alpha \geq 1$, the \emph{$z$-shifted \Renyi} divergence between $\mu$ and $\nu$ is defined as
\[
\Dzalpha{\mu}{\nu} \doteq \inf_{\mu'\colon W_{\infty}(\mu, \mu') \leq z} \Dalpha{\mu'}{\nu}.
\]
\end{defn}
The following follows from the definition:
\begin{proposition}
\label{prop:shifted_renyi}
The shifted \Renyi divergences satisfy the following for any $\mu, \nu$, and any $\alpha \in (1,\infty)$:
  \begin{description}
\item[Monotonicity:] For $0 \leq z \leq z'$, $\Dzpalpha{z'}{\mu}{\nu} \leq \Dzalpha{\mu}{\nu}$.
\item[Shifting:] For any $x\in \cZ$, $\Dzpalpha{\|x\|}{\mu}{\nu} \leq \Dalpha{\mu\conv \mathbf{x}}{\nu}$, where we let $\mathbf{x}$ denote the distribution of the random variable that is always equal to $x$ (note that $\mu\conv \mathbf{x}$ is the distribution of $U+x$ for $U \sim \mu$).
  \end{description}
\end{proposition}

\begin{defn}
\label{def:renyi-max}
For a noise distribution $\zeta$ over a Banach space $(\cZ,\|\cdot\|)$ we measure the magnitude of noise by considering the function that for $a >0$, measures the largest \Renyi divergence of order $\alpha$ between $\zeta$ and the same distribution $\zeta$ shifted by a vector of length at most $a$:
$$R_\alpha(\zeta,a) \doteq \sup_{x\colon \|x\| \leq a} \Dalpha{\zeta \conv \mathbf{x}}{\zeta} .$$
\end{defn}

We denote the standard Gaussian distribution over $\Re^d$ with variance $\sigma^2$ by $\normal{0}{\sigma^2\Id_d}$. By the well-known properties of Gaussians, for any $x\in \Re^d$, and $\sigma$, $\Dalpha{\normal{0}{\sigma^2\Id_d}}{\normal{x}{\sigma^2 \Id_d}} = \alpha \|x\|_2^2 / 2\sigma^2$. This implies that in the Euclidean space, $R_\alpha(\normal{0}{\sigma^2\Id_d},a) = \frac{\alpha a^2}{2\sigma^2}$.

When $U$ and $V$ are sampled from $\mu$ and $\nu$ respectively, we will often abuse notation and write $\Dalpha{U}{V}$, $W_{\infty}(U, V)$ and $\Dzpalpha{z}{U}{V}$ to mean $\Dalpha{\mu}{\nu}$, $W_{\infty}(\mu, \nu)$ and $\Dzpalpha{z}{\mu}{\nu}$, respectively.

\subsection{(\Renyi) Differential Privacy}
\label{sec:privacy_prelims}

The notion of differential privacy (\cref{def:diffPrivacy}) is by now a de facto standard for statistical data privacy~\cite{DMNS06,Dwork06,DR14-book}.  At a semantic level, the privacy guarantee ensures that \emph{an adversary learns almost the same thing about an individual independent of the individual's presence or absence in the data set.} The parameters $(\eps,\delta)$ quantify the amount of information leakage. A common choice of these parameters is $\eps\approx 0.1$ and $\delta=1/n^{\omega(1)}$, where $n$ refers to the size of the dataset.

\begin{defn}[\cite{DMNS06,DKMMN06}]\label{def:diffPrivacy}
	A randomized  algorithm $\A$ is\ed-differentially private (\ed-DP) if, for all neighboring data sets $S$ and $S'$ and for all events $\cO$ in the output space of $\A$, we have
	\[
	\Pr[\A(S)\in \cO] \leq e^{\eps} \Pr[\A(S')\in \cO] +\delta.
	\]
\end{defn}

The notion of neighboring data sets is domain-dependent, and it is commonly taken to capture the contribution of a single individual. In the simplest case $S$ and $S'$ differ in one record, or equivalently, $d_H(S,S') = 1$, where  $d_H(S,S')$ is the Hamming distance. We also define

\begin{defn}[Per-person Privacy]
An algorithm $\A$ operating on a sequence of data points $x_1,\ldots, x_n$ is said to satisfy $\ed$-differentially privacy at index $i$ if for any pair of sequences that differ in the $i$th position, and for any event $\cO$ in the output space of $\A$, we have
  \[
\Pr[\A(x_1,\ldots,x_i,\ldots, x_n) \in \cO] \leq e^{\eps} \Pr[\A(x_1,\ldots,x'_i,\ldots, x_n) \in \cO] + \delta.
\]
\end{defn}

Another related model of privacy is local differential privacy \cite{KLNRS08}. In this model each user executes a differentially private algorithm on their individual input which is then used for arbitrary subsequent computation (we omit the formal definition as it is not used in our work).

Starting with Concentrated Differential Privacy~\cite{DworkRothblum-CDP}, definitions that allow more fine-grained control of the privacy loss random variable have proven useful. The notions of zCDP~\cite{BS16-zCDP}, Moments Accountant~\cite{DLDP}, and \Renyi differential privacy (RDP)~\cite{mironov2017renyi} capture versions of this definition. This approach improves on traditional $(\eps, \delta)$-DP accounting in numerous settings, often leading to significantly tighter privacy bounds as well as being applicable when the traditional approach fails~\cite{PATE, PATE2}. In the current work, we will use the nomenclature based on the notion of the \Renyi divergence (\cref{def:renyi}).

\begin{defn}[\cite{mironov2017renyi}]\label{def:rDiffPrivacy}
	For $1\leq \alpha\leq \infty$ and $\eps \geq 0$,  a randomized  algorithm $\A$ is \emph{$(\alpha,\eps)$-\Renyi differentially private}, or $(\alpha,\eps)$-RDP if for all neighboring data sets $S$ and $S'$ we have
	\[
	\Dalpha{\A(S)}{\A(S')}\leq \eps.
	\]
\end{defn}

Per-person RDP can be defined in an analogous way.
The following two lemmas~\cite{mironov2017renyi} allow translating \Renyi differential privacy to \ed-differential privacy, and give a composition rule for RDP.
\begin{lem}
	\label{lem:rdp_to_dp}
	If $\A$ satisfies $(\alpha,\eps)$-\Renyi differential privacy, then for all $\delta \in (0,1)$ it also satisfies $\left(\eps+\frac{\ln(1/\delta)}{\alpha-1},\delta\right)$-differential privacy. Moreover, pure $(\eps,0)$-differential privacy coincides with $(\infty,\eps)$-RDP.
\end{lem}

The standard composition rule for \Renyi differential privacy, when the  outputs of all algorithms are revealed, takes the following form.

\begin{lem}If $\A_1,\dots,\A_k$ are randomized algorithms satisfying, respectively, $(\alpha,\eps_1)$-RDP,\dots,$(\alpha,\eps_k)$-RDP, then their composition defined as $(\A_1(S),\dots,\A_k(S))$ is $(\alpha,\eps_1+\dots+\eps_k)$-RDP. Moreover, the $i$'th algorithm can be chosen on the basis of the outputs of algorithms $\A_1,\dots,\A_{i-1}$.
\end{lem}

\subsection{Contractive Noisy Iteration}

We start by recalling the definition of a contraction.
\begin{defn}[Contraction]
For a Banach space $(\cZ,\|\cdot\|)$, a function $\eff \colon \cZ \to \cZ$ is said to be {\em contractive} if it is 1-Lipschitz. Namely,
for all $x, y \in \cZ$,
\[
\|\eff(x) - \eff(y)\| \leq \|x - y\|.
\]
\end{defn}

A canonical example of a contraction is projection onto a convex set in the Euclidean space.
\begin{proposition}
\label{prop:proj}
Let $\K$ be a convex set in $\Re^d$. Consider the {\em projection operator}:
\[
\Pi_{\K}(x) \doteq \arg \min_{y \in \K} \|x - y\|.
\]
The map $\Pi_{\K}$ is a contraction.
\end{proposition}
Another example of a contraction, which will be important in our work, is a gradient descent step for a smooth convex function. The following is a standard result in convex optimization~\cite{nesterov-book}; for completeness, we give a proof in \cref{app:contractivity}.
\edef\prop-smooth-contract{\the\value{lem}}
\begin{proposition}\label{prop:smooth-contract}Suppose that a function $f\colon \Re^d \to \Re$ is convex and $\beta$-smooth. Then the function $\eff$ defined~as:
	\[
	\eff(w) \doteq w - \eta \nabla_w f(w)
	\]
	is contractive as long as $\eta \leq 2/\beta$.
\end{proposition}

We will be interested in a class of iterative stochastic processes where we alternate between adding noise and applying some contractive map.
\begin{defn}[Contractive Noisy Iteration (\CNI)]Given an initial random state $X_0 \in \cZ$, a sequence of contractive functions $\eff_t\colon \cZ \to \cZ$, and a sequence of noise distributions $\{\zeta_t\}$, we define the Contractive Noisy Iteration (\CNI) by the following update rule:
\[
X_{t+1} \doteq \eff_{t+1}(X_t) + Z_{t+1},
\]
where $Z_{t+1}$ is drawn independently from $\zeta_{t+1}$.
For brevity, we will denote the random variable output by this process after $T$ steps as $\CNI_T(X_0, \efs, \zetas)$.
\end{defn}

%% file: coupled_descent.tex
\section{Coupled Descent}
\label{sec:coupled_descent}
In this section, we prove a bound on the \Renyi divergence between the outputs of two contractive noisy iterations. Suppose that $X_0$ and $X'_0$ are two random states such that $W_\infty(X_0, X'_0) \leq 1$. The map's contractivity and the fact that we are adding noise $\zeta$ ensures that $X_1$ and $X'_1$ are $R_\alpha(\zeta,1)$-close in $\alpha$-\Renyi divergence.
By the post-processing property of \Renyi divergence, $X_T$ and $X'_T$ are similarly close. Our main theorem says that this can be substantially improved if we do not release the intermediate steps. The noise added in subsequent steps further decreases the \Renyi divergence even when contractive steps are taken in between the noise addition.

While the final result is a statement about \Renyi divergences, the shifted \Renyi divergences play a crucial role in the proof. We start with an important technical lemma that for the noise addition step, allows one to reduce the shift parameter $z$. We will then show how contractive maps affect the shifted divergence. Armed with these results, we prove the main theorem in \cref{ss:main-result}.

\subsection{The Shift-Reduction Lemma}
\label{subsec:shift-reduction}
In this section we prove the key lemma that relates $\Dzalpha{\mu \conv \zeta}{\nu \conv \zeta}$ to $\Dzpalpha{z+a}{\mu}{\nu}$. Recall that we use $R_\alpha(\zeta,a)$ to measure how well noise distribution $\zeta$ hides changes in our norm $\|\cdot\|$ (see \cref{def:renyi-max}):
$$R_\alpha(\zeta,a) \doteq \sup_{x\colon \|x\| \leq a} \Dalpha{\zeta \conv \mathbf{x}}{\zeta} .$$
\begin{lem}[Shift-Reduction Lemma]
\label{lem:shift_reduction}
Let $\mu$,$\nu$ and $\zeta$ be distributions over a Banach space $(\cZ,\|\cdot\|)$.
Then for any $a \geq 0$,
\[
\Dzalpha{\mu \conv \zeta}{\nu \conv \zeta} \leq \Dzpalpha{z+a}{\mu}{\nu} + R_\alpha(\zeta,a).
\]
\end{lem}
\begin{proof}
Let $U$ be distributed as $\mu$ and $V$ as $\nu$. We first show the result for the case when $z=0$. Let $\mu'$ be the distribution certifying $\Dzpalpha{a}{\mu}{\nu}$, that is $\Dalpha{\mu'}{\nu} = \Dzpalpha{a}{\mu}{\nu}$ and $W_{\infty}(\mu, \mu') \leq a$. Let $(U, W)$ be the random variable whose existence is given by Lemma~\ref{lem:winfty_alternate}. That is, $\|W\| \leq a$ with probability $1$, $U \sim \mu$ and $U+W \sim \mu'$. Let $Y$ be an independent random variable distributed as $\zeta$. We can write
\begin{align*}
\Dalpha{\mu \conv \zeta}{\nu \conv \zeta} &= \Dalpha{U+Y}{V+Y}\\
&= \Dalpha{U + W - W + Y}{V+Y}\\
&\leq \Dalpha{(U+W, -W+Y)}{(V, Y)},
\end{align*}
where we have used the post-processing property of \Renyi divergence. Note that the distribution $(V, Y)$ is a product distribution, whereas the factors of $(U+W, -W+Y)$ are dependent.  Denoting the $p_X$ the density function of a random variable $X$, we expand
\begin{align*}
\exp((\alpha - 1)&\Dalpha{(U+W, -W+Y)}{(V, Y)}\\
&= \int \int \left(\frac{p_{(U+W, -W+Y)}(v, y)}{p_{(V, Y)}(v, y)}\right)^\alpha \zeta(y) \nu(v) \ud y \ud v\\
&= \int \int \left(\frac{p_{U+W}(v) \cdot p_{-W+Y \mid U+W=v}(y)}{\nu(v)\cdot \zeta(y)}\right)^\alpha \zeta(y) \nu(v) \ud y \ud v\\
&= \int \left(\frac{p_{U+W}(v)}{\nu(v)}\right)^\alpha \cdot \left(\int \left(\frac{p_{-W+Y \mid U+W=v}(y)}{\zeta(y)}\right)^\alpha \zeta(y)  \ud y \right) \nu(v) \ud v\\
&\leq \int \left(\frac{p_{U+W}(v)}{\nu(v)}\right)^\alpha \nu(v) \ud v \cdot\esssup_{(v',w) \sim p_{(U+W,W)}} \int \left(\frac{p_{-W+Y \mid W=w, U+W=v'}(y)}{\zeta(y)}\right)^\alpha \zeta(y)  \ud y \\
&\leq \int \left(\frac{\mu'(v)}{\nu(v)}\right)^\alpha \nu(v) \ud v \cdot \esssup_{w \sim p_W} \int \left(\frac{p_{-W+Y \mid W=w}(y)}{\zeta(y)}\right)^\alpha \zeta(y)  \ud y\\
&\leq \exp((\alpha - 1) \Dalpha{\mu'}{\nu}) \cdot \exp\left((\alpha - 1)R_\alpha(\zeta,a)\right).\tag{\cref{prop:renyi}}
  \end{align*}
Taking logs and dividing by $(\alpha - 1)$, we get the claim for $z=0$.

The general $z$ case reduces readily to the $z=0$ case. Define
\[
h_z(x) =  \begin{cases}
\hphantom{\frac{x}{\|x\|}}x &\text{if } \|x\| \leq z,\\
\frac{x}{\|x\|} z &\text{otherwise}.
\end{cases}
\]
It is easy to see that $\|h_z(x)\| \leq z$ for all $x$, and that $\|x-h_z(x)\| \leq a$ whenever $\|x\| \leq z+a$.

As before, let $(U, W)$ be r.v.'s from the joint distribution guaranteed by  \cref{lem:winfty_alternate}. Let $W_1 \doteq h_z(W)$ and $W_2 \doteq W - W_1$. It follows that $\|W_1\| \leq z$ and $\|W_2\| \leq a$ with probability 1. We write
\begin{align*}
\Dzalpha{U+Y}{V+Y} &= \Dzalpha{U + W_1 + Y - W_1}{V+Y}\\
&\leq \Dalpha{U + W_1 + Y}{V + Y}\\
&\leq \Dzpalpha{a}{U + W_1}{V} +  R_\alpha(\zeta,a),
\end{align*}
where we have used the $z=0$ case in the last step. On the other hand,
\begin{align*}
\Dzpalpha{a}{U + W_1}{V} &\leq \Dalpha{U+W_1 + W_2}{V}\\
&= \Dalpha{U + W}{V}\\
&= \Dzpalpha{z+a}{U}{V}.
\end{align*}
This completes the proof.
\end{proof}
\subsection{Contractive Maps}
We next show that contractive maps cannot increase a shifted divergence. In the lemma below we give a more general version that allows using different contractive maps.
\begin{lem}[Contraction reduces $\mathrm{D}_{\alpha}^{(z)}$]
\label{lem:gen_contraction_dzalpha}
Suppose that $\eff$ and $\eff'$ are contractive maps on $(\cZ,\|\cdot\|)$ and $\sup_{x} \|\eff(x) - \eff'(x)\| \leq s$. Then for r.v.'s $X$ and $X'$ over $\cZ$,
\[
\Dzpalpha{z+s}{\eff(X)}{\eff'(X')} \leq \Dzalpha{X}{X'}.
\]
\end{lem}
\begin{proof}
By definition of $\Dzalpha{\cdot}{\cdot}$ (see \cref{def:shifted-rdp})and \cref{lem:winfty_alternate}, there is a joint distribution $(X, Y)$ such that $\Dalpha{Y}{X'} = \Dzalpha{X}{X'}$ and $\Pr[\|X - Y\| \leq z] = 1$. By the post-processing property of \Renyi divergence, we have that $\Dalpha{\eff'(Y)}{\eff'(X')} \leq \Dalpha{Y}{X'} = \Dzalpha{X}{X'}$. Moreover,
\begin{align*}
\|\eff(X) - \eff'(Y)\| &\leq \|\eff(X) - \eff(Y)\| + \|\eff(Y) - \eff'(Y)\|\\
&\leq \|X-Y\| + s\\
&\leq z + s.
\end{align*}
Thus $(\eff(X), \eff'(Y))$ is a coupling establishing the claimed upper bound on  $\Dzpalpha{z+s}{\eff(X)}{\eff'(Y)}$.
\end{proof}

\subsection{Privacy Amplification by Iteration}
\label{ss:main-result}
We are now ready to prove our main result. We prove a general statement that can handle changes in several~$\eff$'s; this enables us to easily analyze algorithms that access data points more than once\footnote{Since \Renyi divergence does not satisfy the triangle inequality, blackbox analyses of such algorithms use the group privacy properties of RDP that can be loose.}. Recall that $R_\alpha$ is introduced in \cref{def:renyi-max} and measures the maximal \Renyi divergence of order $\alpha$ between a noise distribution and its shifted copy.

\begin{thm}
\label{thm:pai-general}
Let $X_T$ and $X'_T$ denote the output of $\CNI_T(X_0, \efs, \zetas)$ and $\CNI_T(X_0, \efps, \zetas)$. Let $s_t \doteq \sup_{x} \|\eff_t(x) - \eff'_t(x)\|$. Let $a_1,\ldots, a_T$ be a sequence of reals and let $z_t \doteq \sum_{i \leq t} s_i - \sum_{i \leq t} a_i$. If $z_t \geq 0$ for all $t$, then
\[
\Dzpalpha{z_{T}}{X_T}{X'_T} \leq \sum_{t=1}^{T} R_\alpha(\zeta_t,a_t).
\]
In particular, if $z_T = 0$, then
\[
\Dalpha{X_T}{X'_T} \leq \sum_{t=1}^{T} R_\alpha(\zeta_t,a_t).
\]
\end{thm}
\begin{proof}
The proof is by induction where we use the  contraction-reduces-$\mathrm{D}_{\alpha}^{(z)}$ lemma and then reduce the shift amount by $a_t$ using the shift-reduction lemma.

Let $X_t$ (resp., $X'_t$) denote the $t$'th iterate of the $\CNI(X_0, \efs, \zetas)$ (resp., $\CNI(X_0, \efps, \zetas)$. We argue that for all $t \leq T$,
\[
\Dzpalpha{z_t}{X_t}{X'_t} \leq \sum_{i=1}^{t}  R_\alpha(\zeta_i,a_i).
\]
The base case is $t = 0$. By definition, $X_0=X'_0$ and $z_0=0$.
For the inductive step, let $Z_{t+1}$ denote the random variable drawn from $\zeta_{t+1}$.
\begin{align*}
\Dzpalpha{z_{t+1}}{X_{t+1}}{X'_{t+1}} &=
\Dzpalpha{z_{t+1}}{\eff_{t+1}(X_t) + Z_{t+1}}{\eff'_{t+1}(X'_{t+1}) + Z_{t+1}}&\\
&\leq \Dzpalpha{z_{t+1} + a_{t+1}}{\eff_{t+1}(X_t)}{\eff'_{t+1}(X'_t)} +  R_\alpha(\zeta_{t+1},a_{t+1})\tag{\cref{lem:shift_reduction}}\\
&= \Dzpalpha{z_{t} + s_{t+1}}{\eff_{t+1}(X_t)}{\eff'_{t+1}(X'_t)} +  R_\alpha(\zeta_{t+1},a_{t+1})\tag{Definition of $z_{t+1}$}\\
&\leq \Dzpalpha{z_t}{X_t}{X'_t} + R_\alpha(\zeta_{t+1},a_{t+1})&\tag{\cref{lem:gen_contraction_dzalpha}}\\
&\leq  \sum_{i=1}^{t} R_\alpha(\zeta_i,a_i) + R_\alpha(\zeta_{t+1},a_{t+1}).\tag{induction hypothesis}
  \end{align*}
  This completes the induction step and the proof.
\end{proof}

%% file: privacy.tex
\section{Privacy Guarantees for Noisy Stochastic Gradient Descent}
\label{sec:privacy}
We will now apply our analysis technique to derive the privacy parameters of several versions of the noisy stochastic gradient descent algorithm (also referred to as Stochastic Gradient Langevin Dynamics) defined as follows. We are given a family of convex loss functions over some convex set $\cK \subseteq \Re^d$ parameterized by $x \in \cX$, that is $f(w,x)$ is convex and differentiable in the first parameter for every $x \in \cX$. Given a dataset $S=(x_1,\ldots,x_n)$, starting point $w_0$, rate parameter $\eta$, and noise scale $\sigma$ the algorithm works as follows. Starting from $w_0 \in \K$ perform the following update $v_{t+1}\doteq w_t - \eta (\nabla_w f(w_t,x_{t+1}) + Z)$ and $w_{t+1}\doteq\Pi_\K(v_{t+1})$, where $Z$ is a freshly drawn sample from $\N(0,\sigma^2\Id_d)$ and $\Pi_\K$ denotes the Euclidean projection to set $\K$. We refer to this algorithm as PNSGD$(S,w_0,\eta,\sigma)$ and describe it formally in \cref{alg:ogd}.

\begin{algorithm}[htb]
	\caption{Projected noisy stochastic gradient descent (PNSGD)}
	\begin{algorithmic}[1]
		\REQUIRE Data set $S=\{x_1,\ldots,x_n\}$, $f\colon \K\times\cX\to\Re$ a convex function in the first parameter, learning rate~$\eta$, starting point $w_0\in\K$, noise parameter $\sigma$.
		\FOR{$t\in\{0,\dots,n-1\}$}
			\STATE $v_{t+1}\leftarrow w_t-\eta(\nabla_w f(w_t,x_{t+1})+Z)$, where $Z \sim \cN(0,\sigma^2\Id_d)$.
            \STATE $w_{t+1}\leftarrow \Pi_{\K}\left(v_{t+1}\right)$, where $\Pi_{\K}(w)=\argmin_{\theta\in\K}\|\theta-w\|_2$ is the $\ell_2$-projection on $\K$.
		\ENDFOR
		 \RETURN the final iterate $w_n$.
	\end{algorithmic}
	\label{alg:ogd}
\end{algorithm}

The key property that allows us to treat noisy gradient descent as a contractive noisy iteration is the fact that for any convex function, a gradient step is contractive as long as the function satisfies a relatively mild smoothness condition (see \cref{prop:smooth-contract}). In addition, as is well known, for any convex set $\cK \in \Re^d$, the (Euclidean) projection to $\cK$ is contractive (see \cref{prop:proj}). Naturally, a composition of two contractive maps is a contractive map and therefore we can conclude that PNSGD$(S,w_0,\eta,\sigma)$ is an instance of contractive noisy iteration. More formally, consider the sequence $v_0=w_0,v_1,\ldots,v_n$. In this sequence, $v_{t+1}$ is obtained from $v_t$ by first applying a contractive map that consists of projection to $\K$ followed by the gradient step at $w_t$ and then addition of Gaussian noise of scale $\eta \cdot \sigma$. Note that the final output of the algorithm is $w_n=\Pi_\K(v_n)$ but it does not affect our analysis of divergence as it can be seen as an additional post-processing step.

For this baseline algorithm we prove that points that are used earlier have stronger privacy guarantees due to noise injected in subsequent steps.
\begin{thm}
\label{thm:per-person-privacy}
Let $\K \subseteq \Re^d$ be a convex set and $\{f(\cdot,x)\}_{x\in \cX}$ be a family of convex $L$-Lipschitz and $\beta$-smooth functions over $\K$. Then, for every $\eta \leq 2/\beta, \sigma>0, \alpha > 1$, $t \in [n]$, starting point $w_0\in \K$, and $S\in \cX^n$, PNSGD$(S,w_0,\eta,\sigma)$ satisfies $\left(\alpha,\frac{\alpha\cdot\eps}{n+1-t}\right)$-RDP for its $t$'th input, where $\eps = \frac{2 L^2}{\sigma^2}$.
\end{thm}
\begin{proof}
Let $S\doteq (x_1,\ldots,x_n)$ and $S'\doteq (x_1,\ldots,x_{t-1},x'_t,x_{t+1},\ldots,x_n)$ be two arbitrary datasets that differ at index $t$.
As discussed above, under the smoothness condition $\eta \leq 2/\beta$ the steps of PNSGD$(S,w_0,\eta,\sigma)$ are a contractive noisy iteration. Specifically, on the dataset $S$, the \CNI is defined by the initial point~$w_0$, sequence of functions $g_i(w) \doteq \Pi_\K(w) - \eta \nabla f(\Pi_\K(w), x_i)$ and sequence of noise distributions $\zeta_i \sim \N(0,(\eta \sigma)^2\Id_d)$. Similarly, on the dataset $S'$, the $\CNI$ is defined in the same way with the exception of $g'_t(w) \doteq \Pi_\K(w) - \eta \nabla f(\Pi_\K(w), x'_t)$. By our assumption, $f(w,x)$ is $L$-Lipschitz for every $x\in \cX$ and $w\in \K$ and therefore
  \[
\sup_{w} \|g_t(w) - g'_t(w)\|_2 = \sup_{w} \|\eta \nabla f(\Pi_\K(w), x_t) - \eta \nabla f(\Pi_\K(w), x'_t) \|_2 \leq 2\eta L .
\]
We can now apply \cref{thm:pai-general} with $a_1,\ldots,a_{t-1} = 0$ and $a_t,\ldots,a_n = \frac{2 \eta L}{n-t+1}$. Note that $s_t=2\eta L$ and $s_i=0$ for $i \neq t$. In addition, $z_i \geq 0$ for all $i \leq n$ and $z_n =0$. Hence we obtain that
\[
\Dalpha{X_n}{X'_n} \leq \frac{\alpha}{2 \eta^2 \sigma^2} \sum_{i=1}^{n} a_t^2 \leq  \frac{2 \alpha L^2}{\sigma^2 \cdot (n-t+1)}
\]
as claimed.
\end{proof}
We now consider privacy guarantees for several variants of this baseline approach. These variants are needed to ensure utility guarantees, that require that the algorithm output one of the iterates randomly. Specifically, we define the algorithm Skip-PNSGD$(S,w_0,\eta,\sigma)$ as the algorithm that picks randomly and uniformly $t_0 \in \{0,1,\ldots, \lfloor n/2\rfloor\}$ and then skips the first $t_0$ points. That is, it makes only $n-t_0$ steps and at step $t$ the update is $w'_{t+1} = w_t - \eta (\nabla_w f(w_t,x_{t+1+t_0}) + Z)$. It is easy to see that the privacy guarantees Skip-PNSGD$(S,w_0,\eta,\sigma)$ are at least as good as those we gave for PNSGD$(S,w_0,\eta,\sigma)$ in \cref{thm:per-person-privacy}.
\begin{thm}
\label{thm:per-person-privacy-skip}
Let $\K \subseteq \Re^d$ be a convex set and $\{f(\cdot,x)\}_{x\in \cX}$ be a family of convex $L$-Lipschitz and $\beta$-smooth functions over $\K$. Then, for every $\eta \leq 2/\beta, \sigma>0, \alpha > 1$, $t \in [n]$, starting point $w_0\in \K$, and $S\in \cX^n$, Skip-PNSGD$(S,w_0,\eta,\sigma)$ satisfies $\left(\alpha,\frac{\alpha\cdot\eps}{n+1-t}\right)$-RDP for point with index $t$, where $\eps = \frac{2 L^2}{\sigma^2}$.
\end{thm}

Finally, we consider a version of PNSGD with random stopping. Namely, instead of running for $n$ steps the algorithm picks $T \in [n]$ randomly and uniformly, makes $T$ steps and outputs $w_T$. We refer to this version as Stop-PNSGD$(S,w_0,\eta,\sigma)$. To analyze this algorithm we will need to prove a weak\footnote{The weakness here is the strong (if necessary) assumption that $\Dalpha{p_i}{q_i}\leq c/(\alpha - 1)$ for some $c\leq 1$.} form of convexity for the \Renyi divergence that might have other applications.


\begin{lem}
\label{lem:weak-convexity}
Let $\mu_1,\ldots,\mu_n$ and $\nu_1,\ldots,\nu_n$ be probability distributions over some domain $\cZ$ such that for all $i\in [n]$, $\Dalpha{\mu_i}{\nu_i}\leq c/(\alpha - 1)$ for some $c \in (0,1]$. Let $\rho$ be a probability distribution over $[n]$ and denote by $\mu_\rho$ (or $\nu_\rho$) the probability distribution over $\cZ$ obtained by sampling $i$ from $\rho$ and then outputting a random sample from $\mu_i$ (respectively, $\nu_i$). Then
\[
\Dalpha{\mu_\rho}{\nu_\rho} \leq (1+c) \cdot \Ex_{i\sim \rho}[\Dalpha{\mu_i}{\nu_i}] .
\]
\end{lem}
\begin{proof}
Let $\mu'_\rho$ (or $\nu'_\rho$) be the probability distribution over $[n]\times \cZ$ obtained by sampling $i$ from $\rho$ and then sampling a random $x$ from $\mu_i$ (respectively, $\nu_i$) and outputting $(i,x)$. We can obtain $\mu_\rho$ from $\mu'_\rho$ by applying the function that removes the first coordinate and the same function applied to $\nu'_\rho$ gives $\nu_\rho$. Therefore, by the post-processing properties of the \Renyi divergence, we obtain that $\Dalpha{\mu_\rho}{\nu_\rho} \leq \Dalpha{\mu'_\rho}{\nu'_\rho}$.
Now observe that for every $i\in [n]$ and $x \in \cZ$, $\mu'_\rho(i,x) = \rho(i) \cdot \mu_i(x)$. Therefore,
  \begin{align*}
\Dalpha{\mu'_\rho}{\nu'_\rho} & = \frac{1}{\alpha - 1} \ln \Ex_{(i,x) \sim \nu'_\rho} \left[\left(\frac{\mu'_\rho(i,x)}{\nu'_\rho(i,x)}\right)^\alpha \right]\\
&= \frac{1}{\alpha - 1} \ln \Ex_{i \sim \rho}\left[\Ex_{x \sim \nu_i} \left[\left(\frac{\mu_i(x)}{\nu_i(x)}\right)^\alpha \right] \right] \\
& = \frac{1}{\alpha - 1} \ln \Ex_{i \sim \rho}\left[ e^{(\alpha - 1)\cdot \Dalpha{\mu_i}{\nu_i} } \right] \\
& \leq \frac{1}{\alpha - 1} \ln \Ex_{i \sim \rho}\left[ 1+ (1+c)(\alpha - 1) \cdot \Dalpha{\mu_i}{\nu_i} \right] \\
& = \frac{1}{\alpha - 1} \ln \left(1 + (1+c)(\alpha - 1) \cdot \Ex_{i \sim \rho}\left[\Dalpha{\mu_i}{\nu_i} \right]\right) \\
& \leq \frac{1}{\alpha - 1} \left((1+c)(\alpha - 1)\cdot \Ex_{i \sim \rho}\left[\Dalpha{\mu_i}{\nu_i} \right]\right) \\
& = (1+c) \cdot \Ex_{i \sim \rho}\left[\Dalpha{\mu_i}{\nu_i} \right],
  \end{align*}
where to obtain the inequality in the fourth line we used the fact that for every $a \leq c \leq  1$, $e^a \leq 1+ a + a^2 \leq 1 + (1+c)a$.
\end{proof}

We can now state and prove the privacy guarantees for Stop-PNSGD$(\eta,\sigma)$.
\begin{thm}
\label{thm:privacy-random-stop}
Let $\K \subseteq \Re^d$ be a convex set and $\{f(\cdot,x)\}_{x\in \cX}$ be a family of convex $L$-Lipschitz and $\beta$-smooth functions over $\K$. Then, for every $\eta \leq 2/\beta, \alpha > 1$, starting point $w_0\in \K$, $\sigma \geq L\sqrt{2(\alpha-1) \alpha}$, and dataset $S\in \cX^n$, Stop-PNSGD$(S,w_0,\eta,\sigma)$ satisfies $\left(\alpha,\frac{4 \alpha L^2 \cdot \ln n}{n \sigma^2} \right)$-RDP.
\end{thm}
\begin{proof}
Let $S\doteq(x_1,\ldots,x_n)$ and $S'\doteq(x_1,\ldots,x_{t-1},x'_t,x_{t+1},\ldots,x_n)$ be two arbitrary datasets that differ in the element at index $t$.
For every value of $T \in [n]$, let $X_T$ denote the output of Stop-PNSGD$(S,w_0,\eta,\sigma)$ on $S$ after $T$ steps and analogously define $X'_T$ for Stop-PNSGD$(S',w_0,\eta,\sigma)$. If $t > T$ then the algorithm does not reach $x_t$ (or $x'_t$) and hence $\Dalpha{X_T}{X'_T} =0$. Otherwise, we can use \cref{thm:per-person-privacy} with $n=T$ to obtain that
\[
\Dalpha {X_T}{X'_T} \leq \frac{2 \alpha L^2}{\sigma^2 \cdot (T-t+1)} .
\]
By definition, the output of Stop-PNSGD$(S,w_0,\eta,\sigma)$ corresponds to picking $T$ randomly and uniformly from $[n]$ and then outputting $X_T$. We denote the resulting random variable by $Y_n$ and denote $Y'_n$ the corresponding random variable for $S'$. By our assumption, $\sigma \geq L\sqrt{2(\alpha-1) \alpha}$ and therefore for every $t\geq T$,
\[
\frac{2 \alpha L^2}{\sigma^2 \cdot (T-t+1)} \leq \frac{2 \alpha L^2}{\sigma^2} \leq \frac{1}{\alpha-1}.\]
Hence the conditions of \cref{lem:weak-convexity} are satisfied with $c=1$. This implies that

  \begin{multline*}
\Dalpha {Y_T}{Y'_T} \leq 2 \cdot \frac{1}{n} \sum_{T\in [n]} \Dalpha {X_T}{X'_T} \leq 2 \cdot \frac{1}{n} \sum_{T=t}^{n} \frac{2 \alpha L^2}{\sigma^2 \cdot (T-t+1)}\\
\leq  \frac{4 \alpha L^2 \cdot \ln (n-t+1)}{n \sigma^2} \leq \frac{4 \alpha L^2 \cdot \ln n}{n \sigma^2} .
  \end{multline*}
  \end{proof}

In Appendix~\ref{app:bstproof}, we present a simple analysis of a multiple-pass version of the SGD algorithm. While it gives results that are quantitatively similar to what can be achieved using privacy amplification by sampling results from~\cite{DLDP}, the approach here works in the distributed setting and leads to a significantly simpler proof.

Finally, we remark that PNSGD$(S,w_0,\eta,\sigma)$ and its variants described above satisfy local differential privacy (even without the smoothness assumption). Specifically,
\begin{lem}
\label{lem:local-privacy}
Let $\K \subseteq \Re^d$ be a convex set and $\{f(\cdot,x)\}_{x\in \cX}$ be a family of convex $L$-Lipschitz functions over $\K$. Then, for every $\eta >0, \alpha > 1$, starting point $w_0\in \K$, and dataset $S\in \cX^n$, (Stop/Skip)-PNSGD$(S,w_0,\eta,\sigma)$ satisfies local $\left(\alpha,\frac{2 \alpha L^2}{\sigma^2} \right)$-RDP. In particular, for every $\eps,\delta >0$ and $\sigma = 2L \sqrt{2\ln(1.25/\delta)}/\eps$ it satisfies local $\left(\eps,\delta\right)$-DP.
\end{lem}



%% file: applications.tex
\section{Applications}
\label{sec:applications}
We now show how to use the algorithms we have analyzed to derive new results for privacy-preserving convex optimization. One of the applications we discussed is concerned with a distributed model, where the input records are spread across users' devices. In the ``Our data, ourselves'' model proposed by~\cite{DKMMN06}, each user's device holds their data, and there is no central trusted party. Under reasonable assumptions on the devices, one can simulate a trusted party by means of a Secure Multi-party Computation protocol. While one can assume that all peer-to-peer channels are encrypted, it is reasonable to assume that an attacker can detect the presence or absence of communication. Additionally, in many settings of interest, bandwidth is at a premium and the number of users is large enough that all-to-all communication becomes an implementation bottleneck.

These constraints rule out algorithms that require all parties to be active in every iteration. Consequently, since the presence or absence of communication may be observed by an adversary, we cannot apply privacy amplification by sampling. While algorithms such as bolt-on differential privacy \cite{wu2017bolt} may be usable in the trusted central party setting, their privacy guarantee is uniform and weaker than ours. Our approach gives some baseline local differential privacy and a stronger global privacy guarantee for most users.

\subsection{Private Stochastic Optimization of Smooth Functions}
\label{subsec:optimization}
We will present our results for stochastic convex optimization.  Specifically, let $\K \subseteq \Re^d$ be a convex body contained in a ball of radius $R$ around the origin.  Let $\cP$ be a distribution over convex $L$-Lipschitz functions over $\K$ and let $F(w) \doteq \Ex_{f \sim \cP}[ f(w)]$. We will assume that each data point corresponds to an independent sample from $\cP$ and the goal is to optimize $F(x)$. In order to analyze the performance of the noisy projected gradient descent algorithm for this problem we will need the following classical result about stochastic convex optimization (e.g., \cite{Bubeck15}). For the purposes of this result $F(w)$ can be an arbitrary convex function over $\K$ for which we are given an unbiased stochastic (sub-)gradient oracle $G$. That is for every $w\in \K$, $\Ex[G(w)] \in \partial F(w)$. Let PSGD$(G,w_0,\eta,T)$ denote the execution of the following process: starting from point $w_0$, use the update $w_{t+1} \doteq \Pi_\K(w_t + \eta G(w_t))$ for $t=0,\ldots,T-1$.
\begin{thm}
\label{thm:sco}
Let $\K \subseteq \Re^d$ be a convex body contained in a ball of radius $R$, let $F(w)$ be an arbitrary convex function over $\K$ and let $G$ be an unbiased stochastic (sub-)gradient oracle $G$ for $F$. Assume that for every $w\in \K$, $\Ex[\|G(w)\|_2^2]\leq L_G^2$. For $\eta = 2R/(L_G \sqrt{T})$ and $w_0 \in \K$, let $w_1,\ldots,w_T$ denote the iterates produced by PSGD$(G,w_0,\eta,T)$. Then
\[
\frac{1}{T} \sum_{t\in [T]} \Ex[F(w_t)] \leq F^* + \frac{4RL_G}{\sqrt{T}},
\]
where $F^* \doteq \min_{w\in \K} F(w)$ and the expectation is taken over the randomness of $G$.
\end{thm}
Note that this result gives a bound on the expected value of $F$ averaged over all the iterates. Equivalently, it can be seen as the expected value of $F(w_t)$ with the expectation also taken over $t$ being chosen randomly and uniformly from $[T]$. This corresponds to the random stopping of PSGD$(G,w_0,\eta,T)$. As a result we get the following baseline guarantees for Stop-PNSGD$(S,w_0,\eta,\sigma)$ we defined in \cref{sec:privacy} (namely, these guarantees do not use our amplification analysis and do not require smoothness).

\begin{thm}
\label{thm:random-stop-basic}
Let $\K \subseteq \Re^d$ be a convex body contained in a ball of radius $R$ and $\{f(\cdot,x)\}_{x\in \cX}$ be a family of convex $L$-Lipschitz functions over $\K$. Then for every $\eps>0$, $\delta > 0$, starting point $w_0\in \K$, $\sigma = 2L \sqrt{2\ln(1.25/\delta)}/\eps$, $\eta =  2R/ \sqrt{n (L^2 + d\sigma^2)}$ and dataset $S\in \cX^n$, Stop-PNSGD$(S,w_0,\eta,\sigma)$ satisfies local $\ed$-DP. In addition, if $S$ consists of samples drawn i.i.d.~from an arbitrary distribution $\cP$ over $\cX$, then
\[
\Ex_{S \sim \cP^n}[F(W)] \leq F^* + \frac{4RL}{\sqrt{n}} \cdot \sqrt{1+ \frac{8 d \ln(1.25/\delta)}{\eps^2}},
\]
where $W$ denotes the output of Stop-PNSGD$(S,w_0,\eta,\sigma)$ and $F(w) \doteq \Ex_{x \sim \cP}[ f(w,x)]$.
\end{thm}
\begin{proof}
By \cref{lem:local-privacy}, setting $\sigma \doteq  2L \sqrt{2\ln(1.25/\delta)}/\eps$ ensures local $\ed$-DP. Now we observe that $G(w) = \nabla f(w,x) + Z$ where $x$ is drawn from $\cP$ and $Z$ is drawn from $\N(0,\sigma^2 \Id_d)$ is an unbiased gradient oracle for $F(w)$. Further,
  \[
\Ex[\|G(w)\|^2] = \Ex_{x\sim \cP}[\|\nabla f(w,x)\|^2] + d\sigma^2 \leq L^2 + \frac{8 d L^2 \ln(1.25/\delta)}{\eps^2}.
\]
Hence, we can apply \cref{thm:sco} for $L_G = L \sqrt{1+ \frac{8 d \ln(1.25/\delta)}{\eps^2}}$ and $\eta =  2R/(L_G \sqrt{n})$ to obtain that
\[
\Ex_{S \sim \cP^n}[F(W)] \leq F^* + \frac{4RL}{\sqrt{n}} \cdot \sqrt{1+ \frac{8 d \ln(1.25/\delta)}{\eps^2}}.
\]
\end{proof}

\subsection{Per-person Privacy}
\label{sec:perPersonPrivacy}
We will now show how to combine our stronger privacy guarantees for some of the individuals in the dataset with the utility guarantees in \cref{thm:sco}.
\begin{thm}
\label{thm:per-person-utility}
Let $\K \subseteq \Re^d$ be a convex body contained in a ball of radius $R$ and $\{f(\cdot,x)\}_{x\in \cX}$ be a family of convex $L$-Lipschitz, $\beta$-smooth functions over $\K$. For every $\eps>0$, $\delta > 0$, starting point $w_0\in \K$, $\sigma = 2L \sqrt{2\ln(1.25/\delta)}/\eps$, dataset $S\in \cX^n$ and index $t\in [n]$, if $\eta =  \sqrt{8} R/ \sqrt{n (L^2 + d\sigma^2)} \leq 2/\beta$, then Skip-PNSGD$(S,w_0,\eta,\sigma)$ satisfies local $\ed$-DP and $(\eps/\sqrt{n-t+1},\delta)$-DP at index $t$. In addition, if $S$ consists of samples drawn i.i.d.~from an arbitrary distribution $\cP$ over $\cX$, then
\[
\Ex_{S \sim \cP^n}[F(W)] \leq F^* + \frac{4\sqrt{2}RL}{\sqrt{n}} \cdot \sqrt{1+ \frac{8 d \ln(1.25/\delta)}{\eps^2}},
\]
where $W$ denotes the output of Skip-PNSGD$(S,w_0,\eta,\sigma)$ and $F(w) \doteq \Ex_{x \sim \cP}[ f(w,x)]$.
\end{thm}
\begin{proof}
Our privacy guarantees follow directly from \cref{thm:per-person-privacy} and \cref{lem:local-privacy}. Let us denote by Stop$(n/2)$-PNSGD$(S,w_0,\eta,\sigma)$ the algorithm that runs PNSGD$(S,w_0,\eta,\sigma)$ with a randomly and uniformly chosen stopping time $T\in\{\lceil n/2\rceil,\ldots,n\}$. Observe that the distribution of the output of Skip-PNSGD$(S,w_0,\eta,\sigma)$ on $S\sim \cP^n$ is identical to the output distribution of Stop$(n/2)$-PNSGD$(S,w_0,\eta,\sigma)$ on $S\sim \cP^n$. (This is true since in both cases the starting point, the distribution on the number of steps and the stochastic gradient oracle are identical). Stop$(n/2)$-PNSGD$(S,w_0,\eta,\sigma)$ can be seen as running Stop-PNSGD on $n/2$ points starting from some random point $W_0$ (where $W_0$ is the output of PNSGD on the first $n/2$ points). The utility guarantees for Stop-PNSGD hold for an arbitrary starting point and therefore the utility guarantees for Stop$(n/2)$-PNSGD are the same as those for Stop-PNSGD (\cref{thm:random-stop-basic}) for a dataset consisting of $n/2$ points.
\end{proof}
\subsection{Utility of Public Data}
\label{sec:public_private}
In a variety of settings the algorithm may also have access to a relatively small amount of data from the same distribution that do not require privacy protection. We demonstrate that by using the non-private data points at the end of the training process our per-index privacy guarantees directly lead to substantially improved utility guarantees. In particular, given $\Theta(d\ln(1/\delta)/\eps^2)$ non-private points the utility guarantees of our algorithm match (up to a constant factor) those of non-private learning on the entire dataset.
\begin{cor}
\label{cor:public-private}
Let $\K \subseteq \Re^d$ be a convex body contained in a ball of radius $R$ and $\{f(\cdot,x)\}_{x\in \cX}$ be a family of convex $L$-Lipschitz, $\beta$-smooth functions over $\K$. Let $S_{\mbox{priv}} \in \cX^{n-m}$ and $S_{\mbox{pub}} \in \cX^m$ be two datasets and $S \doteq (S_{\mbox{priv}},S_{\mbox{pub}})$. For every $\eps>0$, $\delta > 0$, starting point $w_0\in \K$, $\sigma = 2L \sqrt{\ln(1.25/\delta)/m}/\eps$, if $\eta =  \sqrt{8}R/ \sqrt{n (L^2 + d\sigma^2)} \leq 2/\beta$, then Skip-PNSGD$(S,w_0,\eta,\sigma)$ satisfies $\ed$-DP relative to $S_{\mbox{priv}}$. In addition, if $S$ consists of samples drawn i.i.d.~from an arbitrary distribution $\cP$ over $\cX$, then
\[
\Ex_{S \sim \cP^n}[F(W)] \leq F^* + \frac{4 \sqrt{2}RL}{\sqrt{n}} \cdot \sqrt{1+ \frac{8 d \ln(1.25/\delta)}{m\eps^2}},
\]
where $W$ denotes the output of Skip-PNSGD$(S,w_0,\eta,\sigma)$ and $F(w) \doteq \Ex_{x \sim \cP}[ f(w,x)]$.
\end{cor}

\subsection{Multiple Convex Optimizations}
\label{sec:multconvOpt}
The privacy guarantees in \cref{thm:privacy-random-stop} do not improve on the $\ed$-DP guarantees for an individual task since in order to convert RDP guarantees to
$\ed$-DP we need to set $\alpha > 1/\eps$ (see \cref{lem:rdp_to_dp}). At the same time, \cref{thm:privacy-random-stop} requires setting $\sigma = \Omega(L\alpha)$ which would give (roughly) the same bound on excess population loss as the one obtained in \cref{thm:random-stop-basic}. When solving $k$ convex optimization tasks on the same dataset, standard analysis requires increasing the noise scale $\sigma$ (and hence the bound on excess loss) by a factor of $\sqrt{k}$ to keep the same \ed-DP level. In contrast, our analysis allows to bound $\ed$-DP directly and only requires increasing $\sigma$ by a factor of $\max \{\tilde{O}(k/n),1\}$. We note that in the context of PAC learning sample complexity of solving multiple learning problems with differential privacy was studied in \cite{BunNS16}. The question of optimizing multiple loss functions was also studied in ~\cite{ullman2015private,FeldmanGV:15}. The bounds given there are incomparable to ours: the multiplicative-weights-update-based approaches there give better bounds when $k \gg n$ and $d$ is small but for $k\leq n$ the bounds given there are worse.

For simplicity of presentation we will state this result for solving a fixed set of $k$ tasks with identical parameters. Composition properties of RDP imply that the bounds can be extended to using problems with different parameters and also allow choosing the tasks in an adaptive way (i.e., after observing the outcome of the previous tasks).
\begin{thm}
\label{thm:mult-conv-utility}
Let $\K \subseteq \Re^d$ be a convex body contained in a ball of radius $R$ and $\{f_i(\cdot,x)\}_{i\in [k], x\in \cX}$ be $k$ families of convex $L$-Lipschitz, $\beta$-smooth functions over $\K$ and $w_0\in \K$ be a starting point. For $\eps\in (0,1)$ and $\delta\in(0,\frac12)$ let $q \doteq \max\left\{\frac{2 k\ln n}{n}, 2\ln(1/\delta) \right\}$, $\sigma \doteq \frac{4L\sqrt{q\ln(1/\delta)}}{\eps}$,
 $\eta \doteq  4 R/ \sqrt{n (L^2 + d\sigma^2)}$. For a dataset $S\in \cX^n$ and $i\in [k]$, let $W_i$ denote the output of Stop-PNSGD$(S,w_0,\eta,\sigma)$ on the $i$'th family of functions (with independent randomness). Then the entire output $(W_1,\ldots,W_k)$ satisfies $\ed$-DP whenever $\eta \leq 2/\beta$. In addition, if $S$ consists of samples drawn i.i.d.~from an arbitrary distribution $\cP$ over $\cX$, then for every $i$,
\[
\Ex_{S \sim \cP^n}[F_i(W_i)] \leq F_i^* + \frac{4RL}{\sqrt{n}} \cdot \sqrt{1+ \frac{16 dq\ln(1/\delta)}{\eps^2}},
\]
where $F_i(w) \doteq \Ex_{x \sim \cP}[ f_i(w,x)]$.
\end{thm}
\begin{proof}
By the composition properties of RDP and \cref{thm:privacy-random-stop}, we have that the output of $k$ executions of Stop-PNSGD$(S,w_0,\eta,\sigma)$ satisfies $\left(\alpha,\frac{4 k \alpha L^2 \cdot \ln n}{n \sigma^2} \right)$-RDP, whenever $\sigma \geq L\sqrt{2(\alpha-1) \alpha}$.
We let $\alpha \doteq \frac{\sigma \sqrt{\ln(1/\delta)}}{L \sqrt{q}}$. Note that this ensures that
  \[
\sigma = \frac{\alpha L \sqrt{q}}{\sqrt{\ln(1/\delta)}} \geq \frac{\alpha L \sqrt{2\ln(1/\delta)}}{\sqrt{\ln(1/\delta)}} = \sqrt{2}\alpha L > L\sqrt{2(\alpha-1) \alpha}.\
\]

Note that for our choice of $\sigma = \frac{4L\sqrt{q\ln(1/\delta)}}{\eps}$ we get that $\alpha = \frac{4\ln(1/\delta)}{\eps} > 2$.

By \cref{lem:rdp_to_dp}, our bound on RDP implies $(\eps,\delta)$-DP as
  \[
\frac{4 k \alpha L^2 \cdot \ln n}{n \sigma^2} +\frac{\ln(1/\delta)}{\alpha -1} < \frac{4 k \cdot \ln(1/\delta) \ln n }{\alpha q n} +\frac{2\ln(1/\delta)}{\alpha} \leq \frac{2\ln(1/\delta)}{\alpha} + \frac{2\ln(1/\delta)}{\alpha} = \eps.
\]

Given the value of $\sigma$, we obtain the bound on the excess population loss from \cref{thm:sco} in the same way as in the proof of \cref{thm:random-stop-basic}.
\end{proof}

\subsection{Removing the Smoothness Assumption}
\label{sec:smoothing}
In this section we show that our assumption on the smoothness of the loss function $f({w},x)$ can effectively be removed in several of our applications. We do this by convolving $f$ with the Gaussian distribution of an appropriate variance. While smoothing a non-smooth objective is a standard technique in optimization (e.g., see \cite{Nesterov:2005,DuchiBW12}) we are not aware of bounds that are stated in the form we need.
Specifically, $f$ may be approximated with its convex Lipschitz extension whose existence and properties are established by the following theorem (its proof is deferred to \cref{app:smoothness}):
\edef\thmtransfer{\the\value{lem}}
\begin{thm}\label{thm:smoo}
	Consider an $L$-Lipschitz convex function $f\colon \K\to\Re$ defined over the convex set $\K\subseteq\Re^d$. For every $\lambda > 0$, there exists a convex function $\hat{f}\colon \Re^d\to\Re$ with the following properties: i) $\hat{f}$ is convex, $L$-Lipschitz, and $L/\lambda$-smooth over $\K$, and ii) for all $w\in\K,|\hat{f}({w})-f({w})|\leq L\lambda\sqrt{d}$.
\end{thm}
In our \cref{thm:per-person-utility,thm:mult-conv-utility} we use $\eta \leq \frac{R\eps}{L\sqrt{n \ln(1/\delta)}}$. Therefore we need our smoothness parameter $\beta \leq 2 \frac{L\sqrt{n\ln(1/\delta)}}{R\eps}$. This means that it suffices  to set $\lambda \doteq \frac{R\eps}{2\sqrt{n\ln(1/\delta)}}$, which by \cref{thm:smoo}, leads to approximation error of $\frac{LR\eps \sqrt{d}}{2\sqrt{n\ln(1/\delta)}}$. Note that this additional error is dominated by the excess population loss whenever $\ln(1/\delta) \geq \eps$ (which is typically the case).
For completeness, we state the immediate corollary of \cref{thm:smoo} for per-person privacy formally.
\begin{cor}
\label{cor:effconvolvedSGDEff}
Let $\K \subseteq \Re^d$ be a convex body contained in a ball of radius $R$ and $\{f(\cdot,x)\}_{x\in \cX}$ be a family of convex $L$-Lipschitz functions over $\K$. For every $\eps>0$, $\delta > 0$, starting point $w_0\in \K$, $\sigma = 2L \sqrt{2\ln(1.25/\delta)}/\eps$, dataset $S\in \cX^n$, $\eta =  \sqrt{8} R/ \sqrt{n (L^2 + d\sigma^2)}$, and index $t\in [n]$, then Skip-PNSGD$(S,w_0,\eta,\sigma)$ executed on $\{f(\cdot,x)\}_{x\in \cX}$ smoothed with $\lambda = \frac{R\eps}{2\sqrt{n\ln(1/\delta)}}$
satisfies  $(\eps/\sqrt{n-t+1},\delta)$-DP at index $t$. In addition, if $S$ consists of samples drawn i.i.d.~from an arbitrary distribution $\cP$ over $\cX$, then
\[
\Ex_{S \sim \cP^n}[F(W)] \leq F^* + \frac{4\sqrt{2}RL}{\sqrt{n}} \cdot \left(\sqrt{1+ \frac{8 d \ln(1.25/\delta)}{\eps^2}} + \frac{\eps \sqrt{d }}{2\ln(1.25/\delta)}\right),
\]
where $W$ denotes the output of Skip-PNSGD$(S,w_0,\eta,\sigma)$ and $F(w) \doteq \Ex_{x \sim \cP}[ f(w,x)]$.
\end{cor}

We remark that in our application that uses public data (\cref{cor:public-private}) the additional error introduced by general smoothing might no longer be dominated by bound on the excess population loss we prove. However, better smoothing techniques can be used for many important classes of functions. For example, generalized linear models can be smoothed with the smoothing error being on the same order as the statistical error (or $LR/\sqrt{n}$). Specifically, these are functions of the form $f(w) = \ell(\langle w,\theta \rangle,y)$ for some parameter $\theta \in \Re^d$ and convex loss function $\ell\colon \Re\times\Re \to \Re$. To smoothen such a function it suffices to convolve $\ell$ with a one-dimensional Gaussian kernel.

%% file: app_contractivity.tex
\section{Contractivity of Gradient Descent for Smooth Functions}\
\label{app:contractivity}

Contractivity of a Gradient Descent step for a smooth convex function is a well-known result in convex optimization (see, e.g., Nesterov~\cite{nesterov-book}). We reproduce a proof below.
\edef\oldtheoremcounter{\the\value{lem}}
\setcounter{lem}{\prop-smooth-contract}
\begin{proposition}Suppose that a function $f\colon \Re^d \to \Re$ is convex, twice differentiable\footnote{This constraint makes the proof simpler but is technically unnecessary.}, and $\beta$-smooth. Then the function $\eff$ defined as:
	\[
	\eff(x) \doteq w - \eta \nabla f(w)
	\]
	is contractive as long as $\eta < 2/\beta$.
\end{proposition}
\setcounter{lem}{\oldtheoremcounter}
\begin{proof}
Let $w, w' \in \Re^d$. We wish to show that
\[
\|\eff(w) - \eff(w') \| \leq \|w - w'\|.
\]
We write
\begin{align*}
\eff(w) - \eff(w') &= w - w' - \eta(\nabla f(w) - \nabla f(w')\\
&= w - w' + \eta(w-w')^\top \nabla^2 f(z)\\
&= (w - w')(\Id - \eta \nabla^2 f(z)),
\end{align*}
for some $z$ on the line joining $w$ and $w'$. By smoothness and convexity, the Hessian has eigenvalues in $[0, \beta]$. Thus,
\[
\|\eff(w) - \eff(w')\| \leq \|w - w'\| \|\Id - \eta \nabla^2 f(z)\|.
\]
Since $0 \preceq \nabla^2 f(z) \preceq \beta \Id$, the claim follows.
\end{proof} 

%% file: app_bstproof.tex
\section{Analyzing Multiple-Epoch SGD}
\label{app:bstproof}

In this section, we show how our techniques can be used to prove privacy for a fixed-ordering version of a multiple-epoch SGD algorithm for minimizing a convex $\beta$-smooth loss function where $\eta \leq 2/\beta$. Formally, we consider the following algorithm:

\begin{algorithm}[htb]
	\caption{Projected noisy multiple-epoch stochastic gradient descent (PNMSGD)}
	\begin{algorithmic}[1]
		\REQUIRE Data set $S=\{x_1,\ldots,x_n\}$, $f\colon\K\times\cX\to\Re$ a function convex in the first parameter, learning rate~$\eta$, starting point $w_0\in\K$, noise parameter $\sigma$.
		\FOR{$j\in\{0,\dots,n-1\}$}
		\FOR{$i\in\{0,\dots,n-1\}$}
		    \STATE $t \leftarrow nj + i$
			\STATE $v_{t+1}\leftarrow w_t-\eta(\nabla_w f(w_t,x_{i+1})+Z)$, where $Z \sim \cN(0,\sigma^2\Id_d)$.
            \STATE $w_{t+1}\leftarrow \Pi_{\K}\left(v_{t+1}\right)$, where $\Pi_{\K}(w)=\argmin_{\theta\in\K}\|\theta-w\|_2$ is the $\ell_2$-projection.
		\ENDFOR
		\ENDFOR
		 \RETURN the final iterate $w_{n^2}$.
	\end{algorithmic}
	\label{alg:multiple-epoch-sgd}
\end{algorithm}

An algorithm similar to this was analyzed by~\cite{bassily2014differentially} who used privacy amplification by sampling to prove that it satisfies $(\eps, \delta)$-DP when $\sigma^2 = \frac{32L^2\ln (n/\delta) \ln (1/\delta)}{\eps^2}$. This analysis can be improved using the techniques of~\cite{DLDP} to ensure that $\sigma = \Theta\left(\frac{L\sqrt{\ln (1/\delta)}}{\eps}\right)$ suffices for a suitable range of $\eps$.

The privacy bound for \cref{alg:multiple-epoch-sgd} follows in a rather straightforward way from \cref{thm:pai-general}.

\begin{thm}
	Under the same assumptions as \cref{thm:per-person-privacy}, projected noisy multiple-epoch SGD (\cref{alg:multiple-epoch-sgd}) satisfies $\left(\alpha,\frac{4\alpha L^2}{\sigma^2}\right)$-RDP.
\end{thm}

\begin{proof}
Let $S$ and $S'$ be two datasets that differ in the $i$'th example. \Cref{alg:multiple-epoch-sgd} run on $S$ (resp.,~$S'$) defines a contractive noise iteration $\CNI(X_0, \efs, \zetas)$ (resp.,~$\CNI(X_0, \efps, \zetas)$). Letting $s_t\doteq 2\eta L$ if $t \equiv i \pmod n$ and 0 otherwise, we observe that $\sup_w \|\eff_t(w) - \eff'_t(w)\|\leq s_t$ for $t\in [n^2]$. 

We set
\[
a_t \doteq \begin{cases}
0 & \textrm{if}\ t < i,\\
\frac{2\eta L}{n} & \textrm{if}\ i \leq t < n(n-1) + i,\\
\frac{2\eta L}{n-i+1} & \textrm{if}\ n(n-1) + i \leq t \leq n^2.
\end{cases}
\]

Recall that  $z_t \doteq \sum_{i \leq t} s_i - \sum_{i \leq t} a_i$ as defined by \cref{thm:pai-general}. It is easy to check that $z_t \geq 0$ for all $t\in[n^2]$ and that $z_{n^2} = 0$. Applying \cref{thm:pai-general} and noting that $\zeta_t = \N(0,(\eta\sigma)^2\Id_d)$ for all $t$, we get that
\begin{align*}
\Dalpha{X_{n^2}}{X'_{n^2}} &\leq \frac{\alpha}{2\eta^2\sigma^2} \cdot \sum_{t=1}^{n^2} a_t^2\\
&= \frac{2\alpha L^2}{\sigma^2} \cdot \left(\frac{n(n-1)}{n^2} + \frac{n-i+1}{(n-i+1)^2}\right)\\
&< \frac{4\alpha L^2}{\sigma^2}.
\end{align*}

It follows that \cref{alg:multiple-epoch-sgd} satisfies $(\alpha, \frac{4\alpha L^2}{\sigma^2})$-RDP as claimed.
\end{proof}

Applying \cref{lem:rdp_to_dp}, with $\alpha = \frac{2 \ln (1/\delta)}{\eps}$ and additionally assuming that $\alpha\geq 5$, we conclude that the projected noisy multiple-epoch SGD satisfies $(\eps, \delta)$-DP for $\sigma = \frac{5 L \sqrt{\ln (1/\delta)}}{\eps}$. Compared to the approach from~\cite{DLDP}, we have a significantly cleaner proof with fewer assumptions on $\sigma$. We remark that the two algorithms differ slightly. Here we fix an ordering and make $n$ passes over the data points in the same order, whereas the algorithm in~\cite{bassily2014differentially} takes $n^2$ steps, each on a uniformly random data point. To obtain utility guarantees for this algorithm one can appeal to standard regret bounds for online algorithms (e.g., see Bubeck~\cite{Bubeck15}). These bounds imply an upper bound on the empirical loss of the randomly chosen iterate. To obtain bounds on the population loss one can appeal to the generalization properties of differential privacy \cite{DworkFHPRR14:arxiv,BassilyNSSSU16}.

%% file: app_smoothness.tex
\section{Smoothing via Convolution with the Gaussian Kernel}\label{app:smoothness}
In this section we prove \cref{thm:smoo}, stated earlier in \cref{sec:smoothing}.

\edef\oldtheoremcounter{\the\value{lem}}
\setcounter{lem}{\thmtransfer}
\begin{thm}[restatement]
	Consider an $L$-Lipschitz convex function $f\colon \K\to\Re$ defined over the convex set $\K\subseteq\Re^d$. For every $\lambda > 0$, there exists a convex function $\hat{f}\colon \Re^d\to\Re$ with the following properties: i) $\hat{f}$ is convex, $L$-Lipschitz, and $L/\lambda$-smooth over $\K$, and ii) for all $w\in\K,|\hat{f}({w})-f({w})|\leq L\lambda\sqrt{d}$.
\end{thm}
\setcounter{lem}{\oldtheoremcounter}
\begin{proof}
	Consider the Gaussian kernel $\zeta = \N(0, \lambda^2\Id_d)$. Before convolving $f$ with this kernel we need to extend $f$ beyond $\K$. Let $h({w})\doteq \min_{v\in\K} f(v)+L\|w-v\|_2$, where $h({w})$ is defined over the complete $\Re^d$. (The function $h$ is also called the convex Lipschitz extension of $f$.) We define the approximation to the function $f(w)$ as $$\hat{f}({w})\doteq \Ex_{Z\sim \zeta}\left[h({w}+Z)\right] .$$ The function $\hat{f}$ satisfies the following properties:
	\begin{enumerate}
		\item {\bf Total on $\Re^d$:} By definition, the function $\hat{f}$ is well-defined on $\Re^d$.
		\item {\bf Lipschitzness and convexity:} Since the function $f({w})$ is convex and $L$-Lipschitz, the Lipschitz extension function $h({w})$ is also convex and $L$-Lipschitz. Hence, $\hat{f}({w})$ is both convex and $L$-Lipschitz as it is defined as a convolution of $h({w})$ with the Gaussian probability kernel.
		\item {\bf Smoothness:} $\hat f({w})$ is $L/\lambda$-smooth. For all $w,w'\in \Re^d$,
		\[
\left\| \nabla \hat{f}(w) - \nabla \hat f(w')\right\|_2 \leq \frac{L}{\lambda} \left\|w-w'\right\|_2.
		\]
Let $p_{Z}$ denote the probability density function of the random variable $Z$. By definition,
		\begin{align*}
\left\| \nabla \hat{f}(w) - \nabla \hat f(w')\right\|_2 & = \left\|\Ex_{Z\sim \N(0, \lambda^2\Id_d)}\left[\partial h(w+Z)\right]-\Ex_{Z\sim\N(0, \lambda^2\Id_d)}\left[\partial h(w'+Z)\right]\right\|_2 \nonumber\\
& = \left\|\Ex_{Z\sim \N(0, \lambda^2\Id_d)}\left[\partial h(w+Z)\right]-\Ex_{Z'\sim \N(w'-w, \lambda^2\Id_d)}\left[\partial h(w+Z')\right]\right\|_2\\
& = \left\|\int_{z}\partial h(w+z)\left(p_{Z}(z)-p_{Z'}(z)\right)\ud z\right\|_2\nonumber\\
		&\leq \sup_{w\in \Re^d} \|\partial h(w)\|_2 \cdot \int_z \left|p_{Z}(z)-p_{Z'}(z)\right|\ud z\\
        & \leq L \cdot 2\ {\sf TV}(p_Z,p_{Z'}),
		\end{align*}
where ${\sf TV}$ refers to the total variation distance. To complete the proof note that by Pinsker's inequality, $${\sf TV}\left(\N(0, \lambda^2\Id_d),\N(w'-w, \lambda^2\Id_d)\right)\leq \sqrt{\frac{1}{2} \Dpalpha{1}{\N(0, \lambda^2\Id_d)}{\N(w'-w, \lambda^2\Id_d)}} \leq \frac{\|w-w'\|_2}{2\lambda} .$$
		\item {\bf Approximation error:} For all ${w}\in\K$, $\left|\hat{f}({w})-f({w})\right|\leq L\lambda\sqrt{d}$.
		By definition, 
\begin{align*}
		\left|\hat f({w})-f({w})\right|&=\left|\Ex_{Z\sim\cN(0,\lambda^2\Id_d)}\left[h(w+Z)-h(w)\right]\right|\\
		&\leq \Ex_{Z\sim\cN(0,\lambda^2\Id_d)}\left[\left|h(w+Z)-h(w)\right|\right]\nonumber\\
		&\leq L \cdot \Ex_{Z\sim\cN(0,\lambda^2\Id_d)}\left[\|Z\|_2\right]\\
		&=L\lambda\sqrt{d}.
		\end{align*}
	\end{enumerate}
Together these properties establish the claim of the theorem.
\end{proof}